\documentclass{article}
\pdfoutput=1
\PassOptionsToPackage{numbers, compress}{natbib}

\usepackage[final]{nips_2018}

\usepackage[utf8]{inputenc} 
\usepackage[T1]{fontenc}    
\usepackage{hyperref}       
\usepackage{wrapfig}       
\usepackage{url}            
\usepackage{booktabs}       
\usepackage{enumitem}
\usepackage{amsfonts}       
\usepackage{amsthm}
\usepackage{nicefrac}       
\usepackage{microtype}      
\usepackage{graphicx}
\usepackage{amsmath}
\usepackage{amssymb}
\usepackage{multirow}
\usepackage{algorithm}
\usepackage{algpseudocode}
\usepackage{caption}
\usepackage{subcaption}
\usepackage{tikz}
\usetikzlibrary{bayesnet}
\usetikzlibrary{matrix}
\usetikzlibrary{calc}
\usetikzlibrary{backgrounds}
\usepackage{bbm}
\usepackage{tabularx}
\newtheorem{proposition}{Proposition}
\newtheorem*{proposition*}{Proposition}

\newcommand{\mcD}{\mathcal{D}}

\newcommand{\mcU}{\mathcal{U}}
\newcommand{\E}{\mathbb{E}}

\newcommand{\reals}{\mathbb{R}}
\DeclareMathOperator*{\supp}{supp}

\DeclareMathOperator*{\KL}{KL}

\DeclareMathOperator*{\RELU}{ReLU}

\DeclareMathOperator*{\softmax}{softmax}
\DeclareMathOperator*{\enc}{enc}

\newcommand{\given}{\,|\,}

\title{Latent Alignment and Variational Attention }

\author{
Yuntian Deng\thanks{Equal contribution.} \And
Yoon Kim$^{*}$ \And
Justin Chiu \And
Demi Guo \And
Alexander M. Rush \And \vspace{-4mm} 
\\
$\texttt{\small \{dengyuntian@seas,yoonkim@seas,justinchiu@g,dguo@college,srush@seas\}.harvard.edu} $ \vspace{4mm} \\
School of Engineering and Applied Sciences \\
Harvard University \\
Cambridge, MA, USA
}

\begin{document}

\maketitle

\begin{abstract}
  Neural attention has become central to many state-of-the-art models
  in natural language processing and related domains. Attention
  networks are an easy-to-train and effective method for softly simulating
  alignment; however, the approach does
  not marginalize over latent alignments in a probabilistic sense. This property
  makes it difficult to compare attention to other alignment
  approaches, to compose it with probabilistic models, and to perform
  posterior inference conditioned on observed data. A related latent
  approach, hard attention, fixes these issues, but is generally
  harder to train and less accurate. This work considers
  \textit{variational attention} networks, alternatives to soft and
  hard attention for learning latent variable alignment models, with
  tighter approximation bounds based on amortized variational
  inference. We further propose methods for reducing the variance of
  gradients to make these approaches computationally
  feasible. Experiments show that for machine translation and visual
  question answering, inefficient exact latent variable models
  outperform standard neural attention, but these gains go away
  when using hard attention based training. On the other hand, variational
  attention retains most of the performance gain but with training speed
  comparable to neural attention.

\end{abstract}

\section{Introduction}

Attention networks \cite{Bahdanau2015} have quickly become the
foundation for state-of-the-art models in natural language
understanding, question answering, speech recognition, image
captioning, and more
\cite{Cho2015,Yang2015,Chorowski2015,Chan2015,Rush2015,Xu2015,Sukhbaatar2015,Rock2016}.
Alongside components such as residual blocks and long-short term
memory networks, soft attention provides a rich neural network
building block for controlling gradient flow and encoding inductive
biases.  However, more so than these other components, which are often
treated as black-boxes, researchers use intermediate attention
decisions directly as a tool for model interpretability
\cite{lei2016rationalizing,alvarez2017causal} or as a factor in final
predictions \cite{gu2016incorporating,shin2017classification}. From
this perspective, attention plays the role of a latent alignment
variable \cite{brown1993mathematics,koehn2007moses}. An alternative
approach, hard attention \cite{Xu2015}, makes this
connection explicit by introducing a latent variable for alignment and
then optimizing a bound on the log marginal likelihood using policy
gradients. This approach generally performs worse (aside from a few exceptions such as \cite{Xu2015}) and is
used less frequently than its soft counterpart.

Still the latent alignment approach remains appealing for several
reasons: (a) latent variables facilitate reasoning about dependencies
in a probabilistically principled way, e.g. allowing composition with
other models, (b) posterior inference provides a better basis for
model analysis and partial predictions than strictly feed-forward
models, which have been shown to underperform on
alignment in machine translation \cite{koehn2017six}, and finally (c)
directly maximizing marginal likelihood may lead to better results.

The aim of this work is to quantify the issues with attention and
propose alternatives based on recent developments in variational
inference. While the connection between variational inference and hard
attention has been noted in the literature \cite{Ba2015,Lawson2017},
the space of possible bounds and optimization methods has not been
fully explored and is growing quickly. These tools allow us to better
quantify whether the general underperformance of hard attention
models is due to modeling issues (i.e. soft attention imbues a better
inductive bias) or optimization issues. 

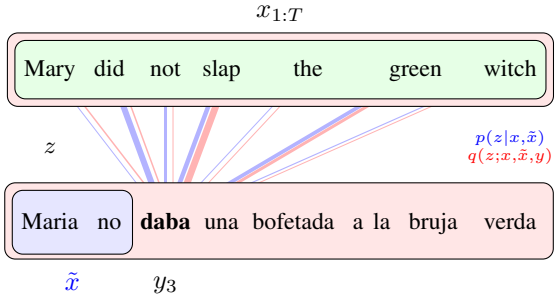
\begin{wrapfigure}{r}{0.5\textwidth}
  \centering
  
  \begin{tikzpicture}[every node/.style={anchor=base,minimum size=8mm}]
    \matrix  (graph) [matrix of nodes, row sep=0.5em,column sep=-0.3em,
    minimum width=0.2em, minimum height=0.5em, font=\small] {
      Mary & 
      did &
      not &
      slap &
      the &
      green &
      witch \\ 
      $z$ & & & & & & $\substack{\textcolor{blue}{p(z| x, \tilde{x})} \\ \textcolor{red}{q(z; x, \tilde{x}, y)} }$ \\
      Maria & 
      no & 
      \textbf{daba} & 
      una & 
      bofetada \; a& 
       la \; bruja & 
      verda\\
    };

    \begin{scope}[on background layer]
      \draw[blue!30, line width=0.1mm] (graph-1-1) -- (graph-3-3);
      \draw[blue!30,line width=0.8mm] (graph-1-2) -- (graph-3-3);
      \draw[blue!30,line width=0.4mm] (graph-1-3) -- (graph-3-3);
      \draw[blue!30,line width=0.6mm] (graph-1-4) -- (graph-3-3);
      \draw[blue!30,line width=0.04mm] (graph-1-5) -- (graph-3-3);
      \draw[blue!30,line width=0.5mm] (graph-1-6) -- (graph-3-3);
      \draw[blue!30,line width=0.15mm] (graph-1-7) -- (graph-3-3);
      
      \draw[draw=red!30, line width=0.2mm] ($(graph-1-1) + (0.1, 0)$) -- ($(graph-3-3) + (0.1, 0)$);
      \draw[draw=red!30, line width=0.2mm] ($(graph-1-2) + (0.1, 0)$) -- ($(graph-3-3) + (0.1, 0)$);
      \draw[draw=red!30, line width=0.1mm] ($(graph-1-3) + (0.1, 0)$) -- ($(graph-3-3) + (0.1, 0)$);
      \draw[draw=red!30, line width=0.9mm] ($(graph-1-4) + (0.1, 0)$) -- ($(graph-3-3) + (0.1, 0)$);
      \draw[draw=red!30, line width=0.1mm] ($(graph-1-5) + (0.1, 0)$) -- ($(graph-3-3) + (0.1, 0)$);
      \draw[draw=red!30, line width=0.4mm] ($(graph-1-6) + (0.1, 0)$) -- ($(graph-3-3) + (0.1, 0)$);
      \draw[draw=red!30, line width=0.1mm] ($(graph-1-7) + (0.1, 0)$) -- ($(graph-3-3) + (0.1, 0)$);

      \draw[rounded corners,fill=red!10] ($ (graph-3-1.north west) +(-0.1,0.1)$) rectangle  node[yshift=-0.8cm]{\textcolor{red}{}} ($(graph-3-7.south east) +(0.1,-0.1)$ ) ;
      
      \draw[rounded corners,fill=red!10] ($ (graph-1-1.north west) +(-0.1,0.1)$) rectangle  node[yshift=0.8cm]{\textcolor{red}{}} ($(graph-1-7.south east) +(0.1,-0.1)$ ) ;
      
      \draw[rounded corners,fill=green!10] (graph-1-1.north west) rectangle  node[ yshift =0.7cm] {$x_{1:T}$} (graph-1-7.south east);
      
      \path[] (graph-3-3.north west) rectangle  node[yshift=-0.9cm]{\textbf{$y_3$}} (graph-3-3.south east);
      
      \draw[rounded corners, fill=blue!10] (graph-3-1.north west) rectangle  node[yshift=-0.9cm]{\textcolor{blue}{$\tilde{x}$}}  ($(graph-3-2.south east) + (-0.1, 0)$);
      
    \end{scope}
  \end{tikzpicture}  
  \caption{\label{fig:sketch} Sketch of variational attention applied to machine translation. 
    Two alignment distributions are shown, the blue prior $p$, 
    and the red variational posterior $q$ taking into account future observations. 
    Our aim is to use $q$ to improve estimates of $p$ and to support improved inference of 
    $z$. 
  }
\end{wrapfigure}

Our main contribution is a \emph{variational attention} approach that
can effectively fit latent alignments while remaining
tractable to train. We consider two variants of variational attention:
\emph{categorical} and \emph{relaxed}.  The categorical method is fit with
amortized variational inference using a learned inference
network and policy gradient with a soft attention variance
reduction baseline. With an appropriate
inference network (which conditions on the entire source/target), it can be used at training time as a drop-in
replacement for hard attention. The relaxed version assumes that the alignment is sampled from a Dirichlet distribution and hence allows attention over multiple source elements. 

Experiments describe how to implement this approach for two major
attention-based models: neural machine translation and visual question
answering (Figure~\ref{fig:sketch} gives an overview of our approach for machine translation). We first show that maximizing exact marginal likelihood can
increase performance over soft attention. We further show that with
variational (categorical) attention, alignment variables significantly
surpass both soft and hard attention results without requiring much
more difficult training. We further explore the impact of posterior
inference on alignment decisions, and how latent variable models might
be employed. Our code is available at \url{https://github.com/harvardnlp/var-attn/}.

\textbf{Related Work} Latent alignment has long been a core problem in
NLP, starting with the seminal IBM models  \cite{brown1992}, HMM-based alignment models
\cite{vogel1996}, and a fast log-linear reparameterization of the IBM 2 model \cite{dyer13}. 
Neural soft attention models
were originally introduced as an alternative approach for neural machine translation \cite{Bahdanau2015}, and have subsequently
been successful on a wide range of tasks (see \cite{Cho2015} for a review of applications). Recent work has combined neural attention with traditional alignment \cite{Cohn2016,Tu2016}
and induced structure/sparsity \cite{Martins2016,Kim2017,Liu2017,Zhu2017,Niculae2017,Niculae2018,Mensch2018}, which can be
combined with the variational approaches outlined in this paper.

In contrast to soft attention models, hard attention
\cite{Xu2015,Ba2015b} approaches use a single sample at training time
instead of a distribution.  These models have proven much more
difficult to train, and existing works typically treat hard attention
as a black-box reinforcement learning problem with log-likelihood as
the reward \cite{Xu2015,Ba2015b,Mnih2015,Gulcehre2016,deng2017image}.
Two notable exceptions are \cite{Ba2015,Lawson2017}: both utilize
amortized variational inference to learn a sampling distribution which
is used obtain importance-sampled estimates of the
log marginal likelihood \cite{Burda2015}. Our method uses uses different
estimators and targets the single sample approach for efficiency,
allowing the method to be employed for NMT and VQA applications.

There has also been significant work in using variational autoencoders
for language and translation application.  Of particular interest are
those that augment an RNN with latent variables (typically Gaussian)
at each time step
\cite{Chung2015,Fraccaro2016,Serban2017,Goyal2017b,Krishnan2017b} and
those that incorporate latent variables into sequence-to-sequence
models \cite{zhang2016variational,Bahuleyan2017,su2018variational,Schulz2018}.
Our work differs by modeling an explicit model component (alignment)
as a latent variable instead of auxiliary latent variables
(e.g. topics). The term
"variational attention" has been used to refer to a different component the output
from attention (commonly called the context vector) as a latent
variable \cite{Bahuleyan2017}, or to model both the memory
and the alignment as a latent variable \cite{Bornschein2018}. Finally, there is some parallel work \cite{wu2018,shankar2018} which also performs exact/approximate marginalization over latent alignments for sequence-to-sequence learning.

\section{Background: Latent Alignment and Neural Attention}

We begin by introducing notation for latent alignment, and then show
how it relates to neural attention. For clarity, we are careful to use
\textit{alignment} to refer to this probabilistic model (Section 2.1), and
\textit{soft} and \textit{hard} attention to refer to two particular 
inference approaches used in the literature to estimate alignment models (Section 2.2).  
\subsection{Latent Alignment}

Figure~\ref{fig:gm}(a) shows a latent alignment model.
Let $x$ be an observed set with associated members
$\{x_1, \ldots, x_i, \ldots, x_T\}$. Assume these are
vector-valued (i.e. $x_i \in \reals^d$) and can be stacked to form a
matrix $X \in \reals^{d \times T}$.  Let the observed
$\tilde{x}$ be an arbitrary ``query''. These generate
a discrete output variable $y \in \mathcal{Y}$. This process is
mediated through a latent alignment variable $z$, which indicates
which member (or mixture of members) of $x$ generates $y$. The
generative process we consider is:
\begin{eqnarray*}
 z \sim {\mcD}(a(x, \tilde{x};\theta )) \ \ \ \  y \sim f(x, z; \theta)
\end{eqnarray*}
where $a$ produces the parameters for an alignment distribution
$\mcD$. The function $f$ gives a distribution over the output, e.g. an
exponential family. To fit this model to data, we set the model
parameters $\theta$ by maximizing the log marginal likelihood of
training examples $(x, \tilde{x}, \hat{y})$:\footnote{When clear
  from context, the random variable is dropped from $\E[\cdot]$. We also interchangeably use $p(\hat{y} \given x, \tilde{x})$ and $f(x, z;\theta)_{\hat{y}}$ to denote $p(y = \hat{y} \given x, \tilde{x})$.}
\begin{eqnarray*}
\max_{\theta}\ \log p(y = \hat{y} \given  x, \ \tilde{x}) &=& \max_{\theta} \ \log \E_z [f(x, z;\theta)_{\hat{y}}]
\end{eqnarray*}

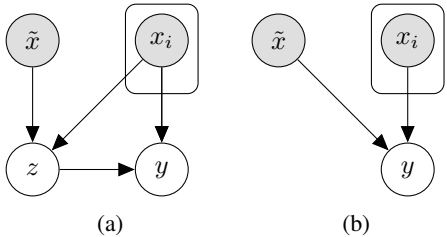
\begin{wrapfigure}{r}{0.5\textwidth}
  \centering
  \begin{subfigure}[]{0.2\textwidth}
  \begin{tikzpicture}
  \node(z)[latent]{$z$};
  \node(y)[right =of z, latent]{$y$};
  \node(x)[obs, above = of y]{$x_i$};
  \node(xp)[above= of z, obs]{$\tilde{x}$};
  \plate {} {(x)} {};
  \edge {x} {y};
  \edge {xp} {z};
  \edge {x} {z};
  \edge {x} {y};
  \edge {z} {y};
  \end{tikzpicture}
  \caption{}
\end{subfigure}
\hspace{0.3cm}
  \begin{subfigure}[]{0.2\textwidth}
  \begin{tikzpicture}
  \node(y)[right =of z, latent]{$y$};
  \node(x)[obs, above = of y]{$x_i$};
  \node(xp)[above= of z, obs]{$\tilde{x}$};
  \plate {} {(x)} {};
  \edge {xp} {y};
  \edge {x} {y};
  \end{tikzpicture}
  \caption{}
\end{subfigure}
  \caption{\label{fig:gm} Models over observed set $x$, query $\tilde{x}$, and alignment $z$. (a) Latent alignment model, (b) Soft attention  with $z$ absorbed into prediction network. }
\end{wrapfigure}

Directly maximizing this log marginal likelihood in the presence of
the latent variable $z$ is often difficult due to the
expectation (though tractable in certain cases).

For this to represent an alignment, we restrict the variable $z$ to be
in the simplex $\Delta^{T-1}$ over source indices $\{1, \dots, T\}$. We consider two
distributions for this variable: first, let $\mcD$ be a
\textit{categorical} where $z$ is a one-hot vector with $z_i = 1$ if $x_i$ is selected.  For example, $f(x,z)$ could use $z$ to pick from
$x$ and apply a softmax layer to predict $y$, i.e.
$f(x, z) = \softmax(\mathbf{W} X z)$ and
$\mathbf{W} \in \reals^{|\mathcal{Y}| \times d}$,
\[\log p(y = \hat{y} \given  x, \ \tilde{x}) = \log \sum_{i=1}^T p(z_i=1 \given x, \tilde{x}) p(y = \hat{y} \given  x, z_i=1) = \log \E_z [\softmax(\mathbf{W} X z)_{\hat{y}}]\]
This computation requires a factor of $O(T)$ additional runtime, and introduces a major computational factor into
already expensive deep learning models.\footnote{Although not our main focus, explicit marginalization is
  sometimes tractable with efficient matrix operations on modern
  hardware, and we compare the variational approach to explicit
  enumeration in the experiments. In some cases it is also possible to
  efficiently perform exact marginalization with dynamic programming
  if one imposes additional constraints (e.g. monotonicity) on the
  alignment distribution \cite{Yu2016,Yu2017,Raffel2017}.}

Second we consider a \textit{relaxed} alignment where $z$ is a mixture
taken from the interior of the simplex by letting $\mcD$ be a Dirichlet. This objective
looks similar to the categorical case, i.e. $\log p(y = \hat{y} \given x, \ \tilde{x}) = \log \E_z
[\softmax(\mathbf{W} X z)_{\hat{y}}]$, but the resulting expectation is intractable to compute
exactly.

\subsection{Attention Models: Soft and Hard}

When training deep learning models with gradient methods, it can be
difficult to use latent alignment directly.  As such, two alignment-like
approaches are popular: \textit{soft attention} replaces the
probabilistic model with a deterministic soft function and
\textit{hard attention} trains a latent alignment model by maximizing a lower bound on the log marginal likelihood (obtained from Jensen's inequality) with policy gradient-style training. We briefly describe how these
methods fit into this notation.  

\paragraph{Soft Attention} Soft attention networks use an altered model shown in Figure~\ref{fig:gm}b. Instead of using a latent
variable, they employ a deterministic network to compute an
expectation over the alignment variable. We can write this model using the same functions $f$ and $a$ from above,
\[ \log p_{\text{soft}}(y\given x, \ \tilde{x}) = \log
  f(x,\E_z[z];\theta) = \log \softmax(\mathbf{W}X\E_z[z])\]

A major benefit of soft attention is efficiency.  Instead of
paying a multiplicative penalty of $O(T)$ or requiring integration,
the soft attention model can compute the expectation before $f$.
While formally a different model, soft attention has been described as
an approximation of alignment \cite{Xu2015}. Since $\E[z] \in \Delta^{T-1}$, soft attention
uses a convex combination of the input representations $X \E[z]$
(the \emph{context vector}) to obtain a distribution over
the output.  While also a ``relaxed'' decision, this expression differs from both the latent alignment
models above. Depending on $f$, the gap between
$\E[f(x, z)]$ and $f(x, \E[z])$ may be large.

However there are some important special cases. In the case where
$p(z \given x, \tilde{x})$ is deterministic, we have
$\E[f(x, z)] = f(x, \E[z])$, and
$p(y \given x, \ \tilde{x})= p_{\text{soft}}(y \given x, \
\tilde{x})$.
In general we can bound the absolute difference based on the maximum
curvature of $f$, as shown by the following proposition.
\begin{proposition}
Define $g_{x, \hat{y}} : \Delta^{T-1} \mapsto [0,1]$ to be the function given by $g_{x, \hat{y}}(z) = f(x, z)_{\hat{y}}$ (i.e. $g_{x, \hat{y}}(z)=p(y=\hat{y} \given x, \tilde{x}, z))$ for a twice differentiable function $f$. Let $H_{g_{x,\hat{y}}}(z)$ be the Hessian of $g_{x,\hat{y}}(z)$ evaluated at $z$, and further suppose $\Vert H_{g_{x,\hat{y}}}(z) \Vert_2 \le c$ for all $z \in \Delta^{T-1}, \hat{y} \in \mathcal{Y}$, and $x$, where $\Vert \cdot \Vert_2$ is the spectral norm. Then for all $\hat{y} \in \mathcal{Y}$,
\[ \, | \, p(y = \hat{y} \given x, \tilde{x}) -  p_\textup{soft}(y = \hat{y} \given x, \tilde{x}) \,| \le c \]
\end{proposition}
The proof is given in Appendix A.\footnote{It is also
  possible to study the gap in finer detail by considering
  distributions over the inputs of $f$ that have high probability
  under approximately linear regions of $f$, leading to the notion of
  \emph{approximately expectation-linear} functions, which was
  originally proposed and studied in the context of dropout
  \cite{Ma2017}.} Empirically the soft approximation works
remarkably well, and often moves towards a sharper
distribution with training. Alignment distributions learned this way
often correlate with human intuition (e.g. word alignment in machine
translation)
\cite{koehn2017six}.\footnote{Another way of viewing soft attention
is as simply a non-probabilistic learned function. While it is
possible that such models encode better inductive biases, our
experiments show that when properly optimized, latent alignment
attention with explicit latent variables do outperform soft
attention.}

\paragraph{Hard Attention} Hard attention is an
approximate inference approach for latent alignment (Figure~\ref{fig:gm}a)
\cite{Xu2015,Ba2015,Mnih2015,Gulcehre2016}. Hard attention takes a
single hard sample of $z$ (as opposed to a soft mixture) and then
backpropagates through the model. The approach is derived by two
choices: First apply Jensen's inequality to get a lower bound on the
log marginal likelihood,
$ \log \E_{z}[p(y \given x, z)] \ge \E_{z}[\log p(y \given x, z)] $, then maximize this lower-bound with policy gradients/REINFORCE  \cite{Williams1992} to obtain
unbiased gradient estimates,
\[\nabla_\theta \E_{z} [\log f(x,z))] = \E_{z} [\nabla_\theta\log f(x,z) + (\log f(x,z) - B) \nabla_\theta \log p(z \given x, \tilde{x}) ], \]
where $B$ is a baseline that can be used to
reduce the variance of this estimator.
To implement this approach efficiently, hard attention uses Monte
Carlo sampling to estimate the expectation in the gradient
computation. For efficiency, a single sample from
$p(z \given x, \tilde{x})$ is used, in conjunction with other tricks to reduce
the variance of the gradient estimator (discussed more below)
\cite{Xu2015,Mnih2014,Mnih2016}.  
\section{Variational Attention for Latent Alignment Models}

Amortized variational inference (AVI, closely related to variational
auto-encoders) \cite{Kingma2014,Rezende2014,Mnih2014} is a class
of methods to efficiently approximate latent variable inference, using
learned inference networks.  In this section we explore this technique
for deep latent alignment models, and propose methods for
\textit{variational attention} that combine the benefits of soft and
hard attention.

First note that the key approximation step in hard attention is to
optimize a lower bound derived from Jensen's inequality. This gap could
be quite large, contributing to poor performance. \footnote{Prior works
on hard attention have generally approached the problem as a
black-box reinforcement learning problem where the rewards are given
by $\log f(x, z)$. Ba et al. (2015) \cite{Ba2015} and Lawson et
al. (2017) \cite{Lawson2017} are the notable exceptions, and both
works utilize the framework from \cite{Mnih2016} which obtains multiple samples from a learned sampling distribution to optimize the IWAE bound
\cite{Burda2015} or a reweighted wake-sleep objective.}  
Variational inference methods directly aim to tighten this gap. In particular, the
\emph{evidence lower bound} (ELBO) is a parameterized bound over a
family of distributions $q(z) \in \mathcal{Q}$ (with the constraint
that the $\supp q(z) \subseteq \supp p(z \given x, \tilde{x}, y)$),
\[ \log \E_{z \sim p(z \given x, \tilde{x})}[p(y \given x, z)] \ge
  \E_{z \sim q(z)}[\log p(y \given x, z)] - \KL[q(z) \, \Vert \, p(z
  \given x, \tilde{x})] \]
This allows us to search over variational distributions $q$ to improve
the bound. It is tight when the variational distribution is equal to
the posterior, i.e.  $q(z) = p(z \given x, \tilde{x}, y)$. Hard
attention is a special case of the ELBO with 
$q(z) = p(z \given x, \tilde{x})$.

There are many ways to optimize the evidence lower bound; an effective
choice for deep learning applications is to use \textit{amortized
  variational inference}. AVI uses an \textit{inference network} to produce the
parameters of the variational distribution $q(z; \lambda)$. The
inference network takes in the input, query, and the output, i.e.
$\lambda = enc(x, \tilde{x}, y ; \phi)$.  The objective aims to
reduce the gap with the inference network $\phi$ while also training the generative model $\theta$,
\[ \max_{\phi, \theta} \E_{z \sim q(z;\lambda)}[\log p(y \given x, z)] - \KL[q(z;\lambda) \, \Vert \, p(z \given x, \tilde{x})]\]
With the right choice of optimization strategy and inference network this form of variational attention can provide a general method for learning latent alignment models. In the rest of this section, we consider strategies for accurately and efficiently computing this objective; in the next section, we describe instantiations of $enc$ for specific domains.

\paragraph{Algorithm 1: Categorical Alignments}
\begin{figure}
  \centering
\begin{minipage}{.48\linewidth}
\begin{algorithm}[H]
  \begin{algorithmic}
    \State{$\lambda \gets \enc(x, \tilde{x}, y ; \phi)$  }
    \Comment{\textit{Compute var. params}}
    \State{$z \sim q(z; \lambda)$ }
    \Comment{\textit{Sample var. attention}}
    \State{$\log f(x,z)$}
    \Comment{Compute output dist\;}
    \State{$z' \gets \E_{ p(z' \given x, \tilde{x})}[z']$ }
    \Comment{Compute soft atten.\;\; }
    \State{$B = \log f(x, z')$ }
    \Comment{Compute baseline dist}
    \State{Backprop $\nabla_{\theta}$ and $\nabla_{\phi}$ based on eq.~\ref{eq:policy} and KL}
  \end{algorithmic}
  \caption{Variational  Attention}
\end{algorithm}
\end{minipage} \hspace{0.2cm}\begin{minipage}{.48\linewidth}
\begin{algorithm}[H]
  \begin{algorithmic}
    \State{$\max_\theta \E_{z\sim p}[\log p(y \given x, z)]$ }
    \Comment{\textit{Pretrain fixed $\theta$}}
    \State{$\ldots$}
    \State{$u \sim \mcU$ }
    \Comment{\textit{Sample unparam. \; \; }}
    \State{$z \gets g_\phi(u)$  }
    \Comment{\textit{Reparam sample \; \; \; }}
    \State{$\log f(x,z)$}
    \Comment{Compute output dist\;}
    \State{Backprop $\nabla_{\theta}$ and $\nabla_{\phi}$, reparam and KL}
  \end{algorithmic}
  \caption{Variational Relaxed Attention}
\end{algorithm}
\end{minipage}
\end{figure}
First consider the case where $\mcD$, the alignment distribution,
and $\cal Q$, the variational family, are categorical distributions. Here the generative
assumption is that $y$ is generated from a single index of $x$. Under
this setup, a low-variance estimator of $\nabla_\theta \text{ELBO}$, is easily obtained through a single sample from
$q(z)$. For $\nabla_\phi \text{ELBO}$, the
gradient with respect to the KL portion is easily computable, but there is
an optimization issue with the gradient with respect to the first term
$\E_{z\sim q(z)} [\log f(x,z))]$.

Many recent methods target this issue, including neural estimates of
baselines \cite{Mnih2014,Mnih2016}, Rao-Blackwellization
\cite{Ranganath2014}, reparameterizable relaxations
\cite{Jang2017,Maddison2017}, and a mix of various techniques
\cite{Tucker2017,Grathwohl2018}. We found that an approach using
REINFORCE \cite{Williams1992} along with a specialized baseline was
effective. However, note that REINFORCE is only one of the inference choices we can select, and as we will show later, alternative approaches such as reparameterizable relaxations work as well. Formally, we first apply the likelihood-ratio trick to obtain an
expression for the gradient with respect to the inference network parameters $\phi$,
\[ \nabla_{\phi} \E_{z \sim q(z)}[\log p(y \given x,
  z)] = \E_{z \sim q(z)} [ ( \log f(x,z) -
  B)\nabla_\phi \log q(z) ] \]
As with hard attention, we take a single Monte Carlo sample (now drawn from the
variational distribution). Variance reduction of this estimate falls
to the baseline term $B$. The ideal (and intuitive) baseline would be
$\E_{z \sim q(z)}[\log f(x,z)]$, analogous to the value function in reinforcement learning. While
this term cannot be easily computed, there is a natural, cheap
approximation: soft attention (i.e.  $\log f(x, \E[z])$). Then the gradient is 
\begin{equation}
 \E_{z \sim q(z)} \left[ \left( \log \frac{ f(x,z) }{ f(x, \E_{z' \sim p(z' \given x, \tilde{x})}[z'])}\right) \nabla_\phi \log q(z \given x, \tilde{x}) \right]  \label{eq:policy}
\end{equation}
Effectively this weights gradients to $q$ based on the ratio of the
inference network alignment approach to a soft attention
baseline. Notably the expectation in the soft attention is over $p$
(and not over $q$), and therefore the baseline is constant with
respect to $\phi$. Note that a similar baseline can also be used for 
hard attention, and we apply it to  both variational/hard attention models in our experiments.

\paragraph{Algorithm 2: Relaxed Alignments}

Next consider treating both $\mcD$ and $\cal Q$ as Dirichlets, where
$z$ represents a mixture of indices. This model is in some sense
closer to the soft attention formulation which assigns mass to
multiple indices, though fundamentally different in that we still formally treat alignment as a latent variable. Again the aim is to find a low variance gradient estimator. Instead of using REINFORCE, certain continuous distributions allow the
use reparameterization \cite{Kingma2014}, where sampling $z \sim q(z)$
can be done by first sampling from a simple unparameterized
distribution $\mathcal{U}$, and then applying a transformation
$g_{\phi}(\cdot)$, yielding an unbiased estimator,
\[ \E_{u \sim \mathcal{U}} \left[\nabla_\phi \log p(y|x,g_\phi(u))\right ] - \nabla_\phi \KL\left [q(z) \, \Vert \, p(z \given x, \tilde{x})\right] \]

The Dirichlet distribution is not directly reparameterizable. While transforming the standard uniform distribution with the inverse CDF of Dirichlet would result in a Dirichlet distribution, the inverse CDF does not have an analytical solution. However, we can use rejection based sampling to get a sample, and employ implicit differentiation to estimate the gradient of the CDF \cite{jankowiak2018pathwise}.

Empirically, we found the random initialization would result in
convergence to uniform Dirichlet parameters for $\lambda$. (We suspect that it is
easier to find low KL local optima towards the center of the
simplex). In experiments, we therefore initialize the latent alignment
model by first minimizing the Jensen bound,
$\E_{z\sim p(z\given x, \tilde{x})}[\log p(y \given x, z)]$, and then introducing the inference network.

\section{Models and Methods}

We experiment with variational attention in two different domains
where attention-based models are essential and widely-used: neural
machine translation and visual question answering.

\paragraph{Neural Machine Translation}

Neural machine translation (NMT) takes in a source sentence
and predicts each word of a target sentence $y_j$ in
an auto-regressive manner. The model first contextually embeds each
source word using a bidirectional LSTM to produce the vectors
$x_1 \ldots x_T$.  The query $\tilde{x}$ consists of an LSTM-based
representation of the previous target words $y_{1:j-1}$. Attention is
used to identify which source positions should be used to
predict the target. The parameters of $\mcD$ are generated
from an MLP between the query and source \cite{Bahdanau2015}, and $f$
concatenates the selected $x_i$ with the query $\tilde{x}$ and passes
it to an MLP to produce the distribution over the next target word
$y_j$.

For variational attention, the inference network applies a
bidirectional LSTM over the source and the target to obtain the hidden
states $x_1, \dots, x_T$ and $h_1, \dots, h_S$, and produces the
alignment scores at the $j$-th time step via a bilinear map,
$s_{i}^{(j)} = \exp(h_j^\top \mathbf{U} x_i)$. For the categorical
case, the scores are normalized,
$q(z^{(j)}_i = 1 ) \propto s_{i}^{(j)}$; in the relaxed case the
parameters of the Dirichlet are $\alpha_i^{(j)} = s_i^{(j)}$. Note,
the inference network sees the entire target (through bidirectional
LSTMs).  The word embeddings are shared between the
generative/inference networks, but other parameters are separate.

\paragraph{Visual Question Answering}
Visual question answering (VQA) uses attention to locate the parts of
an image that are necessary to answer a textual question.  We follow
the recently-proposed ``bottom-up top-down'' attention approach
\cite{Anderson2018}, which uses Faster R-CNN \cite{Ren2015} to obtain
object bounding boxes and performs mean-pooling over the convolutional
features (from a pretrained ResNet-101 \cite{He2016}) in each bounding box to obtain object representations
$x_1, \dots, x_T$. The query $\tilde{x}$ is obtained by running an
LSTM over the question, the attention function $a$ passes the query
and the object representation through an MLP. The prediction function
$f$ is also similar to the NMT case: we concatenate the chosen $x_i$
with the query $\tilde{x}$ to use as input to an MLP which produces a
distribution over the output.  The inference network $enc$ uses the
answer embedding $h_y$ and combines it with $x_i$ and $\tilde{x}$ to
produce the variational (categorical) distribution,
\[ q(z_i = 1) \propto \exp(u^\top \tanh (\mathbf{U}_1 (x_i \odot
  \RELU(\mathbf{V}_1h_y)) + \mathbf{U}_2 (\tilde{x} \odot
  \RELU(\mathbf{V}_2h_y) ))) \]
where $\odot$ is the element-wise product. This parameterization
worked better than alternatives. We did not experiment with the
relaxed case in VQA, as the object bounding boxes already give us
the ability to attend to larger portions of the image.

\paragraph{Inference Alternatives}

For categorical alignments we described maximizing a particular variational lower bound with REINFORCE. Note that other alternatives exist, and we briefly discuss them here: 1) instead of the single-sample variational bound we can use a multiple-sample importance sampling based approach such as Reweighted Wake-Sleep (RWS) \cite{Ba2015} or VIMCO \cite{mnih2016variational}; 2) instead of REINFORCE we can approximate sampling from the discrete categorical distribution with Gumbel-Softmax \cite{jang2016categorical}; 3) instead of using an inference network we can directly apply Stochastic Variational Inference (SVI) \cite{hoffman2013stochastic} to learn the local variational parameters in the posterior.

\paragraph{Predictive Inference}

At test time, we need to marginalize out the latent variables, i.e.
$\E_z [p(y \given x, \tilde{x}, z)]$ using $p(z \given x, \tilde{x})$.  In the
categorical case, if speed is not an issue then enumerating
alignments is preferable, which incurs a multiplicative cost of $O(T)$
(but the enumeration is parallelizable). Alternatively we experimented
with a $K$-max renormalization, where we only take the top-$K$
attention scores to approximate the attention distribution (by re-normalizing). This makes the multiplicative cost constant with
respect to $T$. 
For the relaxed case, sampling is necessary.

\section{Experiments}
\paragraph{Setup}
For NMT we mainly use the IWSLT dataset \cite{Cettolo2017}. This dataset is
relatively small, but has become a standard benchmark for experimental
NMT models. We follow the same preprocessing as in \cite{Edunov2017}
with the same Byte Pair Encoding vocabulary of 14k tokens \cite{sennrich2016}. To show that variational attention scales to large datasets, we also experiment on the WMT 2017 English-German dataset \cite{W17-4700}, following the preprocessing in \cite{vaswani2017attention} except that we use newstest2017 as our test set. For VQA, we use the VQA
2.0 dataset. As we are interested in intrinsic evaluation (i.e. log-likelihood)
in addition to the standard VQA metric, we randomly select half of the
standard validation set as the test set (since we need access to the
actual labels).\footnote{ VQA eval metric is defined as
  $\min\{ \frac{\# \text{ humans that said answer }}{3}, 1\}$. Also
  note that since there are sometimes multiple answers for a given
  question, in such cases we sample (where the sampling probability is proportional to the number of humans that said the answer) to get a single
  label.} (Therefore the numbers provided are not strictly comparable
to existing work.) While the preprocessing is the same as
\cite{Anderson2018}, our numbers are worse than previously reported as
we do not apply any of the commonly-utilized techniques to improve
performance on VQA such as data augmentation and label smoothing.

Experiments vary three components of the systems: (a) training
objective and model, (b) training approximations, comparing enumeration or
sampling,\footnote{Note that enumeration does not imply exact if we
  are enumerating an expectation on a lower bound.} (c) test
inference.  All neural models have the same 
architecture and the exact same number of parameters $\theta$ (the
inference network parameters $\phi$ vary, but are not used at test).
When training hard and variational attention with sampling both use the same baseline,
i.e the output from soft attention. The full architectures/hyperparameters for both NMT and VQA are given in Appendix B.

\paragraph{Results and Discussion}

Table~\ref{tab:eval} shows the main results. We
first note that hard attention underperforms soft attention, even
when its expectation is enumerated. This indicates that Jensen's
inequality alone is a poor bound. On the other hand, on both experiments, 
exact marginal likelihood outperforms soft attention, indicating 
that when possible it is better to have latent alignments.

For NMT, on the IWSLT 2014 German-English task, variational attention with enumeration and sampling performs
comparably to optimizing the log marginal likelihood, despite
the fact that it is optimizing a lower bound. We believe that this is
due to the use of $q(z)$, which conditions on the entire source/target and therefore
potentially provides better training signal to $p(z \, | \, x, \tilde{x})$ through the KL
term. Note that it is also
possible to have $q(z)$ come from a pretrained external model, such as a traditional alignment model \cite{dyer13}.
Table~\ref{tab:ext} (left) shows
these results in context compared to the best reported values for this
task. Even with sampling, our system improves on the state-of-the-art. On the larger WMT 2017 English-German task, the superior performance of variational attention persists: our baseline soft attention reaches 24.10 BLEU score, while variational attention reaches 24.98. Note that this only reflects a reasonable setting without exhaustive tuning, yet we show that we can train variational attention at scale. 
For VQA the trend is largely similar, and results for NLL with
variational attention improve on soft attention and hard
attention. However the task-specific evaluation metrics are slightly
worse.

\begin{table}
\small
  \centering
  \begin{tabular}{llcrrrrrrrr}
    \toprule
         & & & \multicolumn{2}{c}{NMT}  & &\multicolumn{2}{c}{VQA} \\
    Model & Objective & $\E$   & PPL &  BLEU & &NLL &   Eval \\
    \midrule
  Soft Attention& $\log p(y \given \E[z])$ & - & 7.17 &  32.77& & 1.76& 58.93 \\
  Marginal Likelihood & $\log \E[p]$ &  Enum & 6.34  &33.29&   & 1.69 & 60.33\\
  Hard Attention  & $\E_p[\log p]$ & Enum & 7.37 &  31.40&  &  1.78 & 57.60 \\
  Hard Attention  & $\E_p[\log p]$  & Sample &  7.38 &   31.00&  & 1.82 &  56.30 \\
   Variational Relaxed Attention& $\E_q[\log p] - \KL$ & Sample & 7.58	 & 30.05 & & -& - \\
  Variational Attention  & $\E_q[\log p] - \KL$ & Enum & 6.08&  33.68 & &1.69& 58.44\\
  Variational Attention &  $\E_q[\log p] - \KL$ &Sample& 6.17&  33.30  &  & 1.75 & 57.52\\
    \bottomrule
  \end{tabular}
\vspace{2mm}
  \caption{Evaluation on NMT and VQA for the various models. $\E$ column indicates
  whether the expectation is calculated via enumeration (Enum) or a single sample (Sample) during training.
  For NMT we evaluate intrinsically on perplexity (PPL) (lower is better)
  and extrinsically on BLEU (higher is better), where for BLEU we perform beam search with beam size 10 and length penalty (see Appendix B for further details). For VQA we evaluate intrinsically on negative log-likelihood (NLL) (lower is better)
  and extrinsically on VQA evaluation metric (higher is better). All results except for relaxed attention use enumeration at test time. 
\vspace*{-0.8cm}
  }
   \label{tab:eval}
  \end{table}
\begin{table}
\small 
  \begin{minipage}{.6\linewidth}
\small
  \begin{tabular}{lcccc}
    \toprule
     & \multicolumn{2}{c}{PPL} & \multicolumn{2}{c}{BLEU} \\
    Model & Exact & $K$-Max & Exact & $K$-Max \\
    \midrule
   
  Marginal Likelihood & 6.34& 6.90 & 33.29 & 33.31  \\
  Hard + Enum  & 7.37 & 7.37 & 31.40 & 31.37 \\
  Hard + Sample & 7.38 & 7.38 & 31.00 &31.04 \\
  Variational + Enum & 6.08  & 6.42 & 33.68 & 33.69\\
  Variational + Sample & 6.17 & 6.51 & 33.30 & 33.27\\
    \bottomrule
  \end{tabular}
   \end{minipage}
  \begin{minipage}{.5\linewidth}
  \includegraphics[width=0.8\textwidth]{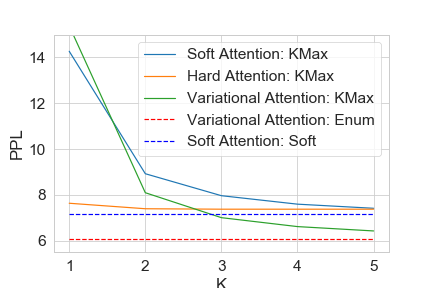}
   \end{minipage}
        \caption{
         (Left) Performance change on NMT from exact decoding to $K$-Max decoding with $K =5$. (see section 5 for definition of
         K-max decoding). (Right) \label{fig:sample} Test perplexity of different approaches while varying $K$ to estimate $\E_z[p(y | x, \tilde{x})]$. Dotted lines compare soft baseline and variational with full enumeration.}
   \label{tab:evalb}
   \vspace*{-0.6cm}
\end{table}

Table~\ref{tab:evalb} (left) considers test inference for variational
attention, comparing enumeration to $K$-max with $K=5$.  For all
methods exact enumeration is better, however $K$-max is a reasonable
approximation. Table~\ref{tab:evalb} (right) shows the PPL of different
models as we increase $K$. Good performance requires
$K > 1$, but we only get marginal benefits for $K >
5$. Finally, we observe that it is possible
to \emph{train} with soft attention and \emph{test} using $K$-Max with a small performance
drop (\texttt{Soft KMax} in Table~\ref{tab:evalb} (right)).
This possibly indicates that soft attention models are approximating
latent alignment models. On the other hand, training with latent alignments 
and testing with soft attention performed badly.

Table~\ref{tab:ext} (lower right) looks at the entropy of the prior
distribution learned by the different models. Note that hard attention has very low
entropy (high certainty) whereas soft attention is quite high. The variational
attention model falls in between. Figure~\ref{fig:attn} (left)
illustrates the difference in practice.

Table~\ref{tab:ext} (upper right) compares inference alternatives for variational attention. RWS reaches a comparable performance as REINFORCE, but at a higher memory cost as it requires multiple samples. Gumbel-Softmax reaches nearly the same performance and seems 
like a viable alternative; although we found its performance is sensitive to its temperature parameter. We also trained a non-amortized SVI model, but found that at similar runtime it was not able to produce satisfactory results, likely due to insufficient updates of the local variational parameters. A hybrid method such as semi-amortized inference \cite{Krishnan2017,kim2018semi} might be a potential future direction worth exploring.

Despite extensive experiments, we found that variational relaxed attention performed worse than other methods. In particular we found that when training with a Dirichlet KL, it is hard to reach low-entropy regions of the simplex, and the attentions are more uniform than either soft or variational categorical attention. Table~\ref{tab:ext} (lower right) quantifies this issue. 
We experimented with other distributions such as Logistic-Normal and Gumbel-Softmax \cite{Jang2017,Maddison2017} but neither fixed this issue. Others have also noted difficulty in training 
Dirichlet models with amortized inference \cite{srivastava2017autoencoding}. 

Besides performance, an advantage of these models is the ability to
perform posterior inference, since the $q$ function can be used
directly to obtain posterior alignments. Contrast this with hard attention
where $q = p(z \given x, \tilde{x})$, i.e. the variational posterior
is independent of the future information.  Figure~\ref{fig:attn} shows
the alignments of $p$ and $q$ for variational attention over a fixed
sentence (see Appendix C for more examples). We see that $q$ is able to use future information to correct
alignments. We note that the inability of soft and hard attention to
produce good alignments has been noted as a major issue in NMT
\cite{koehn2017six}. While $q$ is not used directly in left-to-right
NMT decoding, it could be employed for other applications such as in an
iterative refinement approach \cite{Novak2016,Lee2018}.

\begin{figure}[t]
  \centering
  \begin{subfigure}[b]{0.49\textwidth}
  \centering
  \includegraphics[height=5.4cm]{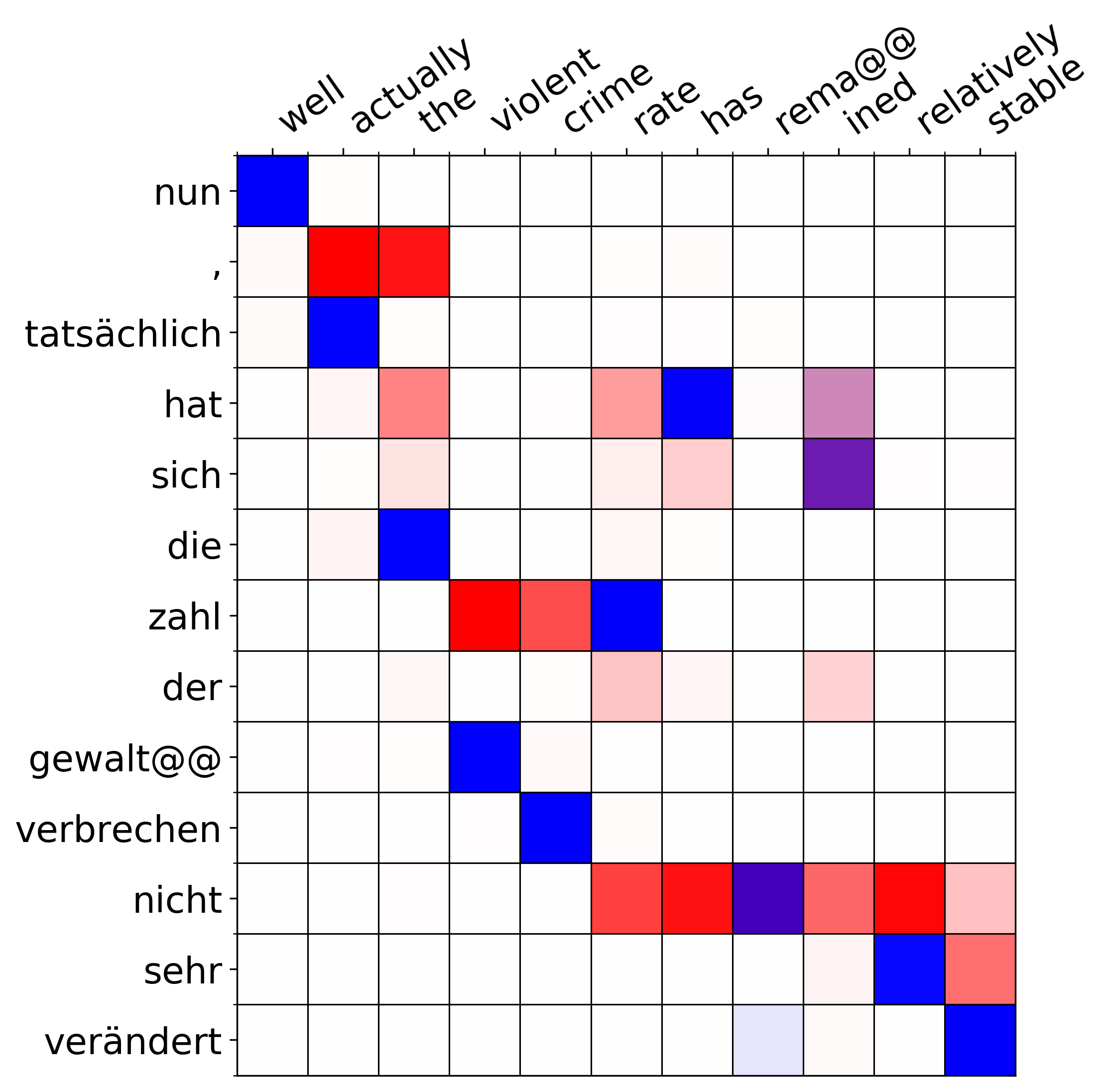}
 
  \label{fig:pq:a}
  \end{subfigure}
  \begin{subfigure}[b]{0.49\textwidth}
  \centering
  \includegraphics[height=5.4cm]{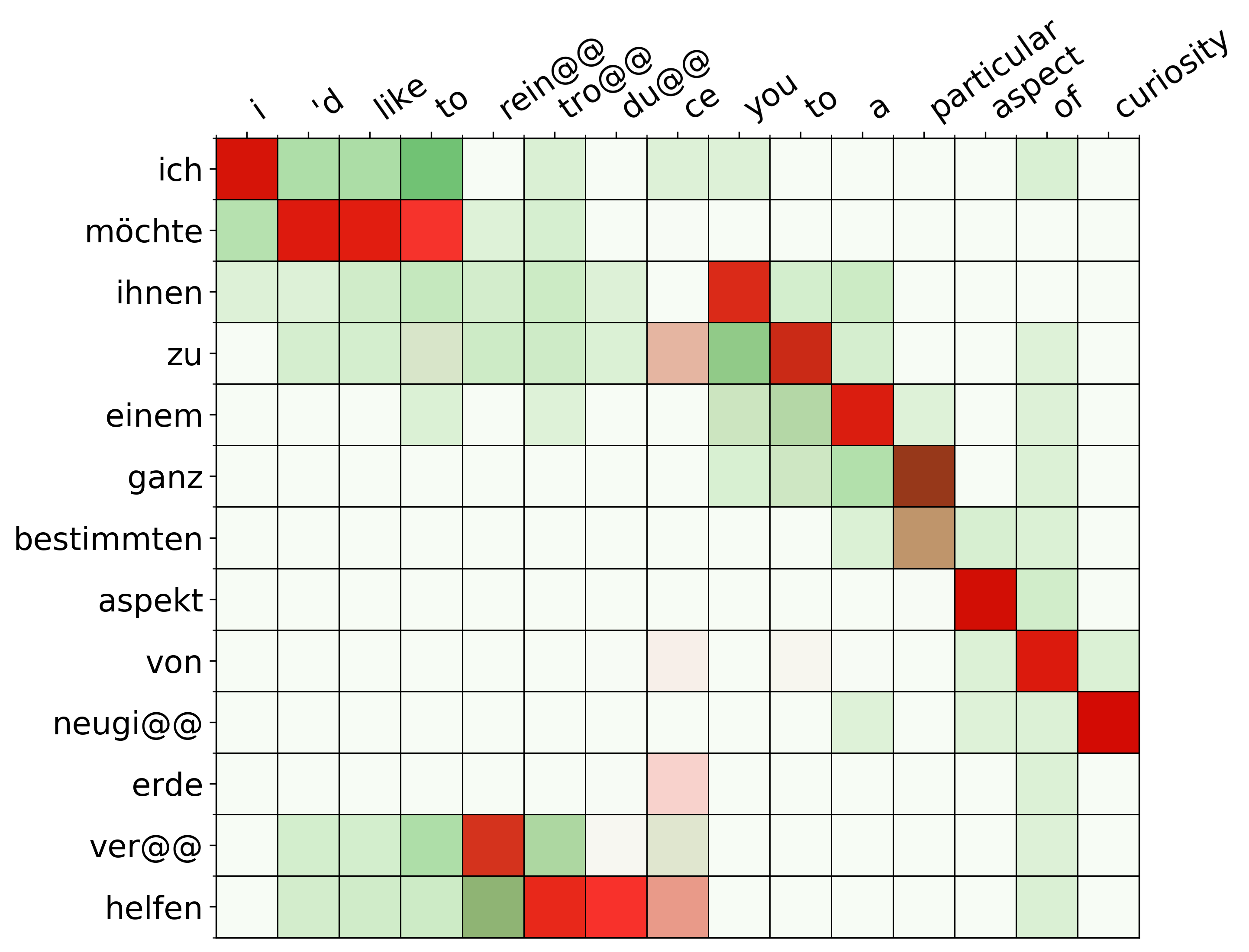}
  \label{fig:pq:b}
  \end{subfigure}
  
  \caption{\label{fig:attn} (Left) An example demonstrating the difference between the prior alignment (red) and the variational posterior (blue) when translating from DE-EN (left-to-right). 
Note the improved blue alignments for \texttt{actually} and \texttt{violent} which benefit 
from seeing the next word. (Right) Comparison of soft attention (green) with the $p$ of variational attention (red). Both models imply a similar alignment, but variational attention has lower entropy.}
  \label{fig:pq}
\end{figure}

\begin{table}
\vspace*{-0.1cm}
\small
\centering
\begin{minipage}{.43\linewidth}
   \begin{tabular}{lc}
       \toprule\\
       & IWSLT \\
    Model & BLEU \\
    \midrule
Beam Search Optimization \cite{Wiseman2016} & 26.36 \\
Actor-Critic \cite{Bahdanau2017} & 28.53 \\
Neural PBMT + LM \cite{Huang2018} & 30.08 \\
Minimum Risk  Training \cite{Edunov2017} & 32.84 \\
     \midrule
  Soft Attention&  32.77\\
  Marginal Likelihood &33.29\\
  Hard Attention + Enum  & 31.40 \\
  Hard Attention + Sample &  30.42 \\
    Variational Relaxed Attention & 30.05 \\
  Variational Attention + Enum & 33.69 \\
  Variational Attention + Sample & 33.30 \\
    \bottomrule
   \end{tabular}
      \end{minipage}\hspace*{0.5cm}
      \begin{minipage}{.43\linewidth}
      \centering
         \begin{tabular}{lccc}
       \toprule
    Inference Method & \#Samples & PPL & BLEU \\
    \midrule
  REINFORCE & 1 & 6.17 & 33.30 \\
  RWS & 5 & 6.41 & 32.96 \\
  Gumbel-Softmax  & 1 & 6.51 & 33.08\\
    \bottomrule
   \end{tabular}
   \begin{tabular}{lcc}
       \toprule
       & \multicolumn{2}{c}{Entropy} \\
    Model & NMT & VQA \\
    \midrule
  Soft Attention&  1.24 & 2.70 \\
  Marginal Likelihood & 0.82 & 2.66 \\
  Hard Attention + Enum  & 0.05 & 0.73\\
  Hard Attention + Sample &  0.07 & 0.58\\
  Variational Relaxed Attention & 2.02 & -\\
  Variational Attention + Enum & 0.54 & 2.07\\
  Variational Attention + Sample & 0.52 & 2.44 \\
    \bottomrule
   \end{tabular}
   \end{minipage}

   \caption{\label{tab:ext} (Left) Comparison against the best prior work for NMT on the IWSLT 2014 German-English test set. (Upper Right) Comparison of inference alternatives of variational attention on IWSLT 2014. (Lower Right) Comparison of different models in terms of implied discrete entropy (lower = more certain alignment).}
   \vspace*{-0.6cm}
   \end{table}

\paragraph{Potential Limitations}

While this technique is a promising alternative to soft attention,
there are some practical limitations: (a) Variational/hard attention
needs a good baseline estimator in the form of soft attention. We
found this to be a necessary component for adequately training the
system.  This may prevent this technique from working when $T$ is
intractably large and soft attention is not an option.  (b) For some
applications, the model relies heavily on having a good posterior
estimator. In VQA we had to utilize domain structure for the inference network construction.  (c) Recent models
such as the Transformer \cite{vaswani2017attention}, utilize many
repeated attention models. For instance the current best translation
models have the equivalent of 150 different attention queries per word
translated. It is unclear if this approach can be used at that scale as 
predictive inference becomes combinatorial.
 
\section{Conclusion}

Attention methods are ubiquitous tool for areas
like natural language processing; however they are difficult to 
use as latent variable models. This work explores
alternative approaches to latent alignment, through variational attention with promising result. Future work
will experiment with scaling the method on larger-scale tasks and in more
complex models, such as multi-hop attention models, transformer models, and structured models, as well as utilizing these latent variables for interpretability and as a way to incorporate prior knowledge.

\section*{Acknowledgements}
We are grateful to Sam Wiseman and Rachit Singh for insightful comments and discussion, as well as Christian Puhrsch for help with translations. This project was supported by a Facebook Research Award (Low Resource NMT). YK is supported by a Google AI PhD Fellowship. YD is supported by a Bloomberg Research Award. AMR gratefully acknowledges the support of NSF CCF-1704834 and an Amazon AWS Research award.

{
\bibliographystyle{plain}
\small
\bibliography{nips_2018}}
\newpage
\vspace{5mm}
{\centering
{\LARGE \textbf{Supplementary Materials for  \\
\vspace{3mm}
\hspace{9mm} Latent Alignment and Variational Attention}}}
\vspace{1cm}
\section*{Appendix A: Proof of Proposition 1}
\begin{proposition*}
Define $g_{x, \hat{y}} : \Delta^{T-1} \mapsto [0,1]$ to be the function given by $g_{x, \hat{y}}(z) = f(x, z)_{\hat{y}}$ (i.e. $g_{x, \hat{y}}(z)=p(y=\hat{y} \given x, \tilde{x}, z))$ for a twice differentiable function $f$. Let $H_{g_{x,\hat{y}}}(z)$ be the Hessian of $g_{x,\hat{y}}(z)$ evaluated at $z$, and further suppose $\Vert H_{g_{x,\hat{y}}}(z) \Vert_2 \le c$ for all $z \in \Delta^{T-1}, \hat{y} \in \mathcal{Y}$, and $x$, where $\Vert \cdot \Vert_2$ is the spectral norm. Then for all $\hat{y} \in \mathcal{Y}$,
\[ \, | \, p(y = \hat{y} \given x, \tilde{x}) -  p_\textup{soft}(y = \hat{y} \given x, \tilde{x}) \,| \le c \]
\end{proposition*}
\begin{proof}
We begin by performing Taylor's expansion of $g_{x,\hat{y}}$ at $\E[z]$:
\begin{align*}
\E[g_{x, \hat{y}}(z)] &= \E\Big[g_{x,\hat{y}}(\E[z]) + (z - \E[z])^\top \nabla g_{x,\hat{y}}(\E[z]) + \frac{1}{2} (z-\E[z])^\top H_{g_{x,\hat{y}}}(\hat{z})(z-\E[z])\Big] \\
&= g_{x,\hat{y}}(\E[z]) + \frac{1}{2}\E [(z-\E[z])^\top H_{g_{x,\hat{y}}}(\hat{z})(z-\E[z])]
\end{align*}
for some $\hat{z} = \lambda z + (1-\lambda) \E[z], \lambda \in [0, 1]$.
Then letting $u = z-\E[z]$, we have
\begin{align*}
|\, (z-\E[z])^\top H_{g_{x,\hat{y}}}(\hat{z}) (z-\E[z])\,|  &= | \, \Vert u \Vert_2^2 \frac{u^\top}{\Vert u \Vert_2} H_{g_{x,\hat{y}}}(\hat{z})  \frac{u}{\Vert u \Vert_2} \,|\\
& \le \Vert u \Vert_2^2 \, c\\
\end{align*}
where $c = \max \{ |\lambda_{\max}|, |\lambda_{\min}| \}$ is the largest absolute eigenvalue of $H_{g_{x,\hat{y}}}(\hat{z})$. (Here $\lambda_{\max}$ and $\lambda_{\min}$
are maximum/minimum eigenvalues of $ H_{g_{X,q}}(\hat{z})$).
Note that $c$ is also equal to the spectral norm $\Vert H_{g_{X,q}}(\hat{z}) \Vert_2$ since the Hessian is symmetric.

Then,
\begin{align*}
| \, \E [(z-\E[z])^\top H_{g_{x,\hat{y}}}(\hat{z})(z-\E[z])] \, | &\le
\E [\, |(z-\E[z])^\top H_{g_{x,\hat{y}}}(\hat{z})(z-\E[z]) \, |]  \\
&\le \E [\Vert u \Vert_2^2 c] \\
&\le 2c\\
\end{align*}
Here the first inequality follows due to the convexity of the absolute value function and the last inequality follows since
\begin{align*}
\Vert u \Vert_2^2  &= (z-\E[z])^\top(z-\E[z]) \\
& = z^\top z + \E[z]^\top \E[z] - 2\E[z]^\top z \\
&\le z^\top z +  \E[z]^\top \E[z] \\
& \le 2
\end{align*}
where the last two inequalities are due to the fact that $z, \E[z] \in \Delta^{T-1}$. Then putting it all together we have,
\begin{align*}
\, | \, p(y = \hat{y} \given x, \tilde{x}) -  p_\textup{soft}(y = \hat{y} \given x, \tilde{x}) \,| &= |\,\E[g_{x, \hat{y}}(z)] - g_{x,\hat{y}}(\E[z]) \,| \\
&= \frac{1}{2} | \, \E [(z-\E[z])^\top H_{g_{x,\hat{y}}}(\hat{z})(z-\E[z])] \, |  \\
&\le c
\end{align*}
\end{proof}

\section*{Appendix B: Experimental Setup}
\subsection*{Neural Machine Translation}

For data processing we closely follow the setup in \cite{Edunov2017}, which uses Byte Pair Encoding over the combined source/target training set to obtain a vocabulary size of 14,000 tokens. However, different from \cite{Edunov2017} which uses maximum sequence length of 175, for faster training we only train on sequences of length up to 125. 

The encoder is a two-layer bi-directional LSTM with 512 units in each direction, and the decoder as a two-layer LSTM with with 768 units. For the decoder, the convex combination of source hidden states at each time step from the attention distribution is used as additional input at the next time step.  Word embedding is 512-dimensional. 

The inference network consists of two bi-directional LSTMs (also two-layer and 512-dimensional each) which is run over the source/target to obtain the hidden states at each time step. These hidden states are combined using bilinear attention \cite{Luong2015} to produce the variational parameters. (In contrast the generative model uses MLP attention from \cite{Bahdanau2015}, though we saw little difference between the two parameterizations).
Only the word embedding is shared between the inference network and the generative model.

Other training details include: batch size of 6, dropout rate of 0.3, parameter initialization over a uniform distribution $\mathcal{U}[-0.1, 0.1]$, gradient norm clipping at 5, and training for 30 epochs with Adam (learning rate = 0.0003, $\beta_1 = $ 0.9, $\beta_2 = $ 0.999) \cite{Kingma2015} with a learning rate decay schedule which starts halving the learning rate if validation perplexity does not improve. Most models converged well before 30 epochs. 

For decoding we use beam search with beam size 10 and length penalty $\alpha = 1$, from \cite{Wu2016}. The length penalty added about 0.5 BLEU points across all the models.
\subsection*{Visual Question Answering}
The model first obtains object features by mean-pooling the pretrained ResNet-101 features \cite{He2016} (which are 2048-dimensional) over object regions given by Faster R-CNN \cite{Ren2015}.The ResNet features are kept fixed and not fine-tuned during training. We fix the maximum number of possible regions to be 36. For the question embedding we use a one-layer LSTM with 1024 units over word embeddings. The word embeddings are 300-dimensional and initialized with GloVe \cite{pennington2014glove}.
The generative model produces a distribution over the possible objects  via applying MLP attention, i.e.
\[ p(z_i = 1 \given x, \tilde{x}) \propto \exp(w^\top \tanh (\mathbf{W}_1 x_i + \mathbf{W}_2\tilde{x})) \]
The selected image region is concatenated with the question embedding and fed to a one-layer MLP with ReLU non-linearity and 1024 hidden units. 

The inference network produces a categorical distribution over the image regions by interacting the answer embedding $h_y$ (which are 256-dimensional and initialized randomly) with the question embedding $\tilde{x}$ and the image regions $x_i$,
\[ q(z_i = 1) \propto \exp(u^\top \tanh (\mathbf{U}_1 (x_i \odot
  \RELU(\mathbf{V}_1h_y)) + \mathbf{U}_2 (\tilde{x} \odot
  \RELU(\mathbf{V}_2h_y) ))) \]
where $\odot$ denotes element-wise multiplication. The generative/inference attention MLPs have 1024 hidden units each (i.e. $w, u \in \reals^{1024}$).
 
Other training details include: batch size of 512, dropout rate of 0.5 on the penultimate layer (i.e. before affine transformation into answer vocabulary), and training for 50 epochs with with Adam (learning rate = 0.0005, $\beta_1 = $ 0.9, $\beta_2 = $ 0.999) \cite{Kingma2015}. 

In cases where there is more than one answer for a given question/image pair, we randomly sample the answer, where the sampling probability is proportional to the number of humans who gave the answer. 
\newpage
\section*{Appendix C: Additional Visualizations}
\begin{figure}[H]
  \centering
  \begin{subfigure}[b]{0.49\textwidth}
  \centering
  \includegraphics[height=5cm]{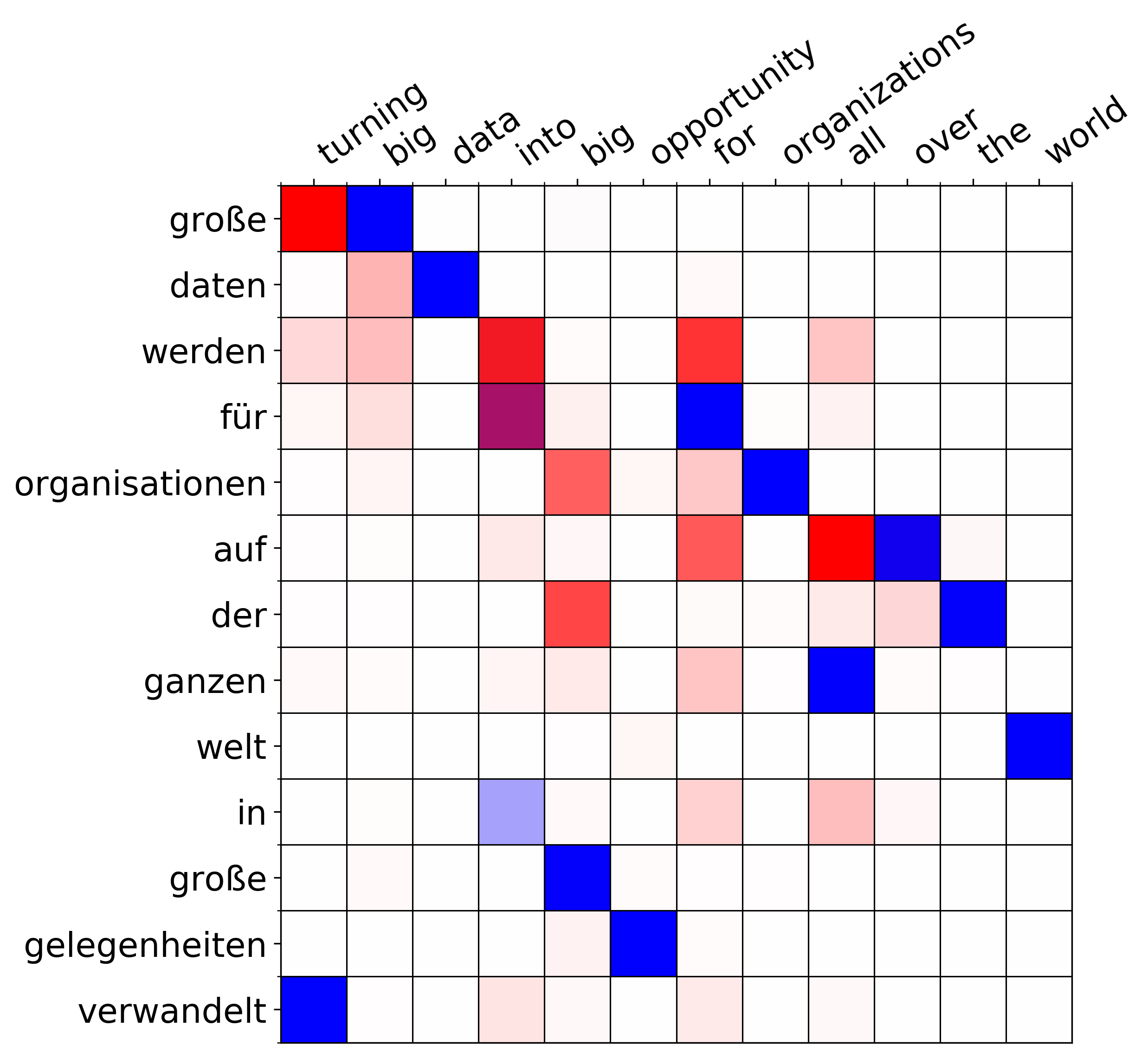}
  \label{fig:a2:a}
  \caption{}
  \end{subfigure}
  \begin{subfigure}[b]{0.49\textwidth}
  \centering
  \includegraphics[height=5cm]{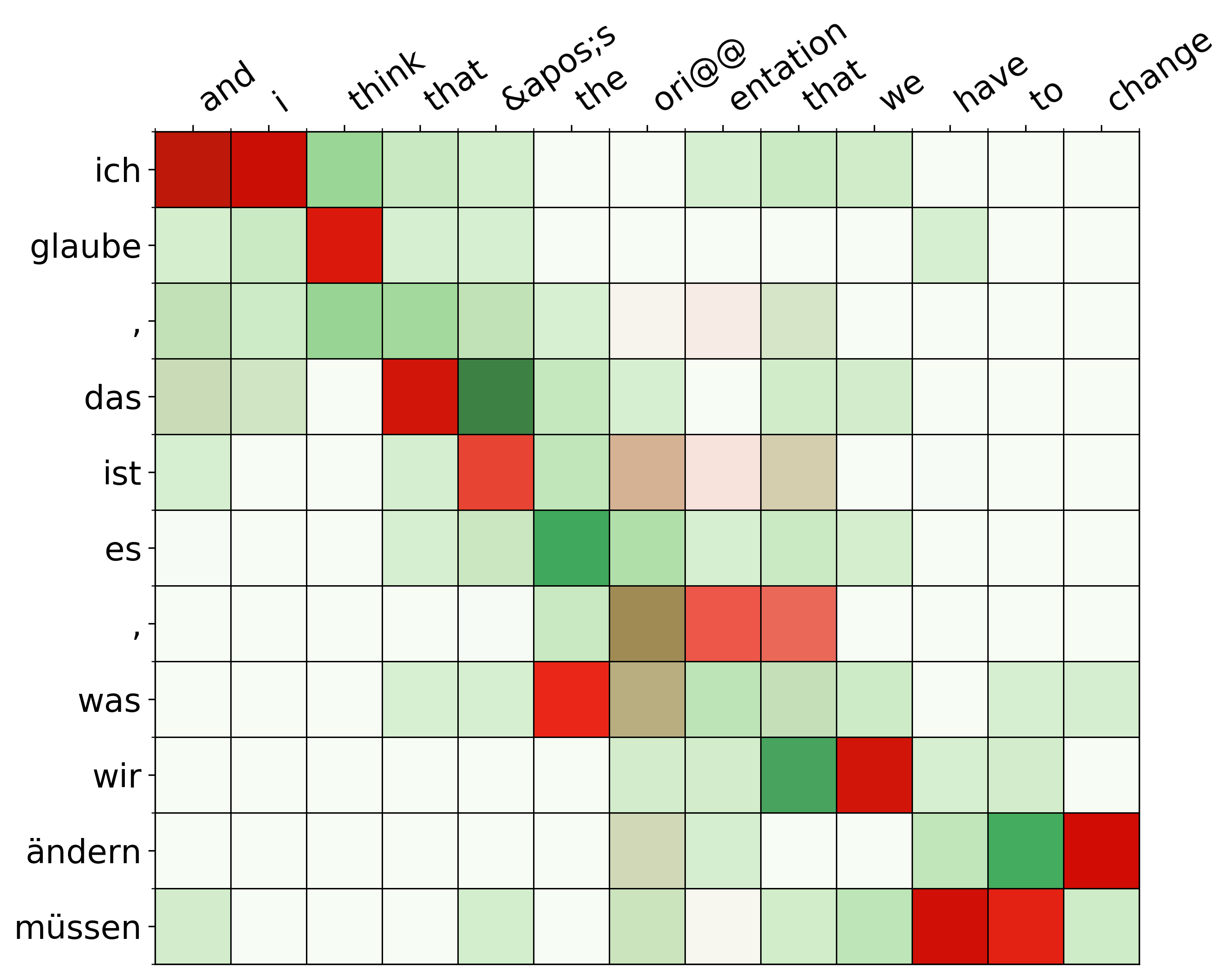}
  \label{fig:a2:b}
  \caption{}
  \end{subfigure}
  \medskip
  \begin{subfigure}[b]{0.49\textwidth}
  \centering
  \includegraphics[height=5cm]{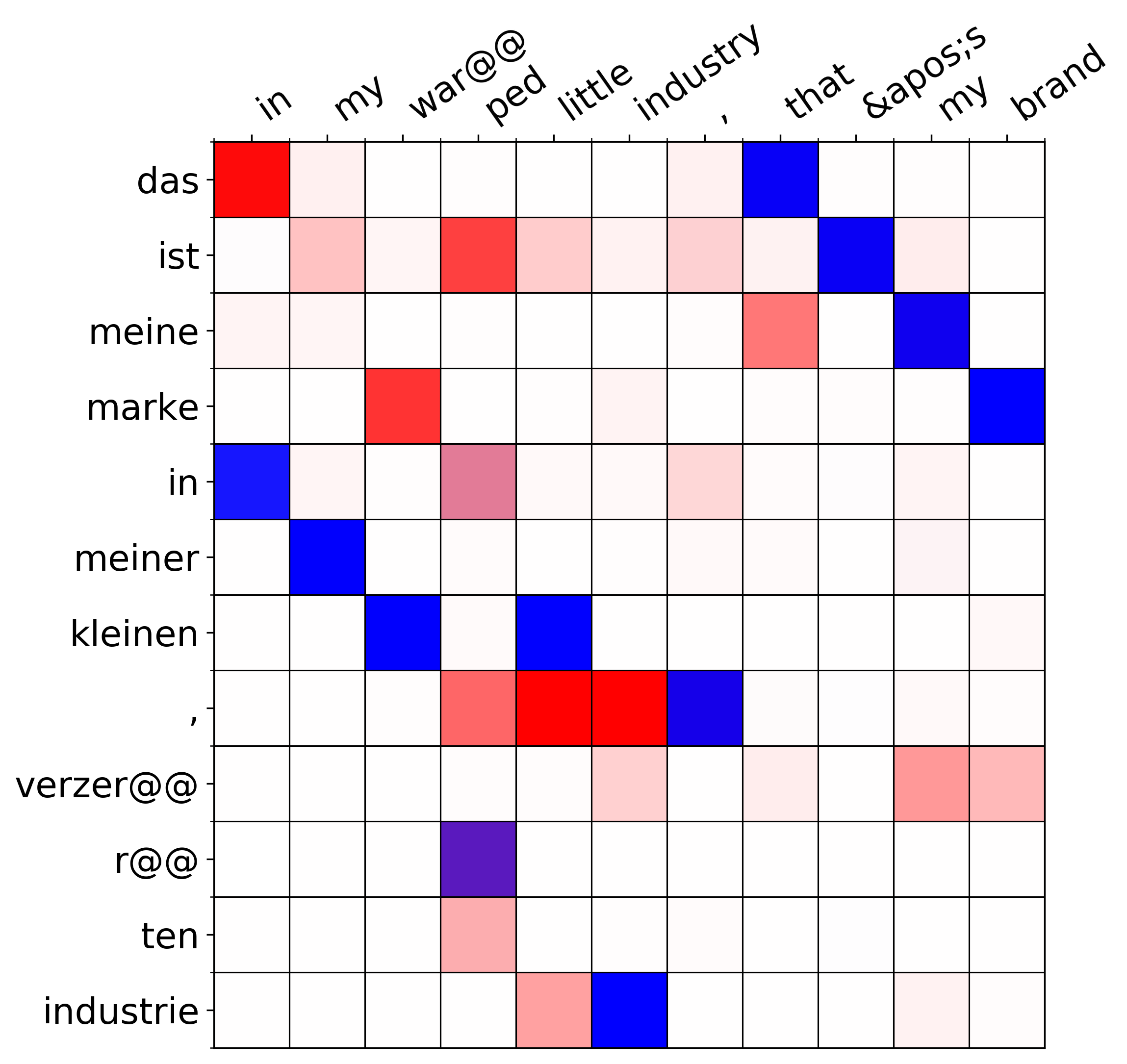}
  \label{fig:a2:c}
  \caption{}
  \end{subfigure}
  \begin{subfigure}[b]{0.49\textwidth}
  \centering
  \includegraphics[height=5cm]{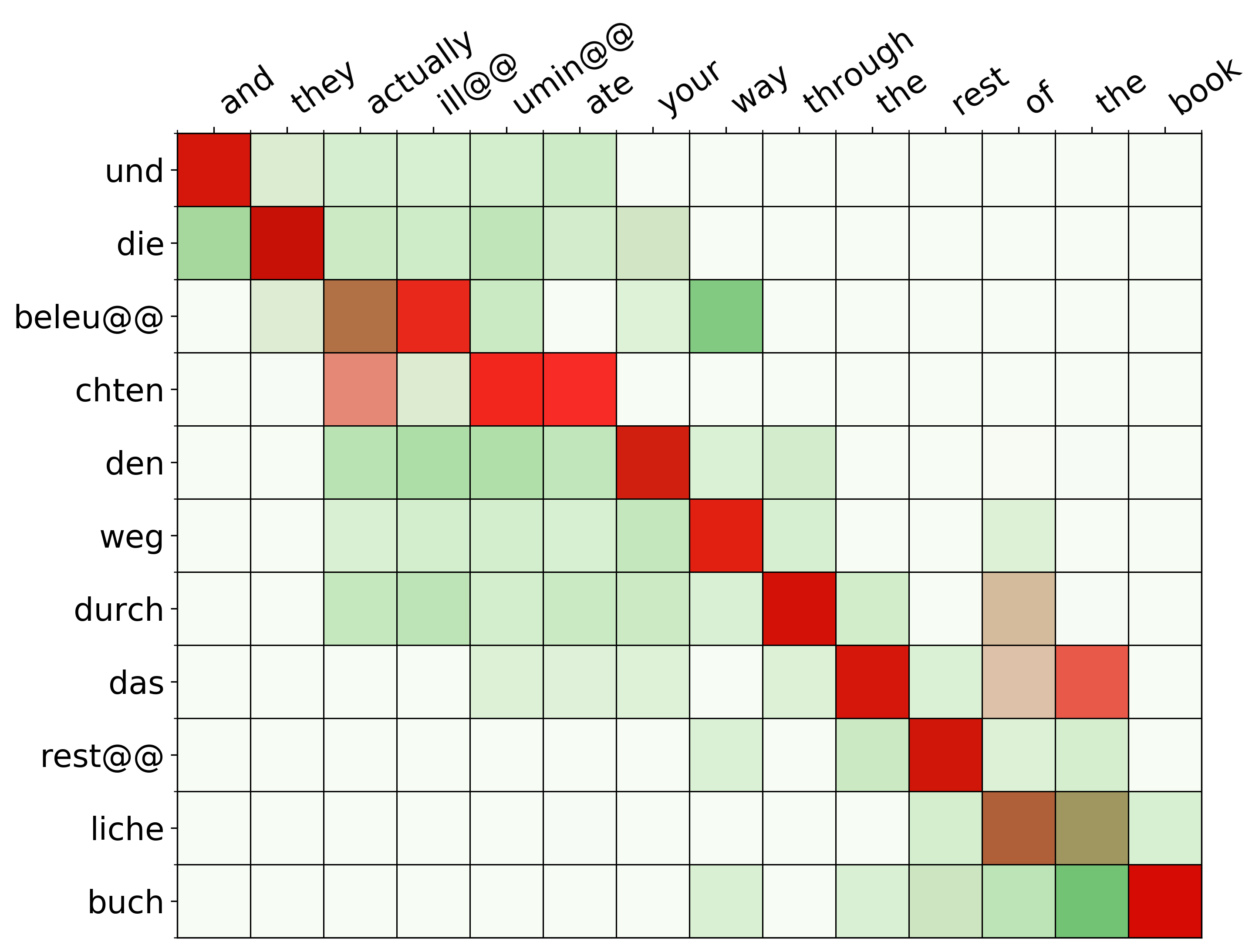}
  \label{fig:a2:d}
  \caption{}
  \end{subfigure}
  \medskip
  \begin{subfigure}[b]{0.49\textwidth}
  \centering
  \includegraphics[height=5cm]{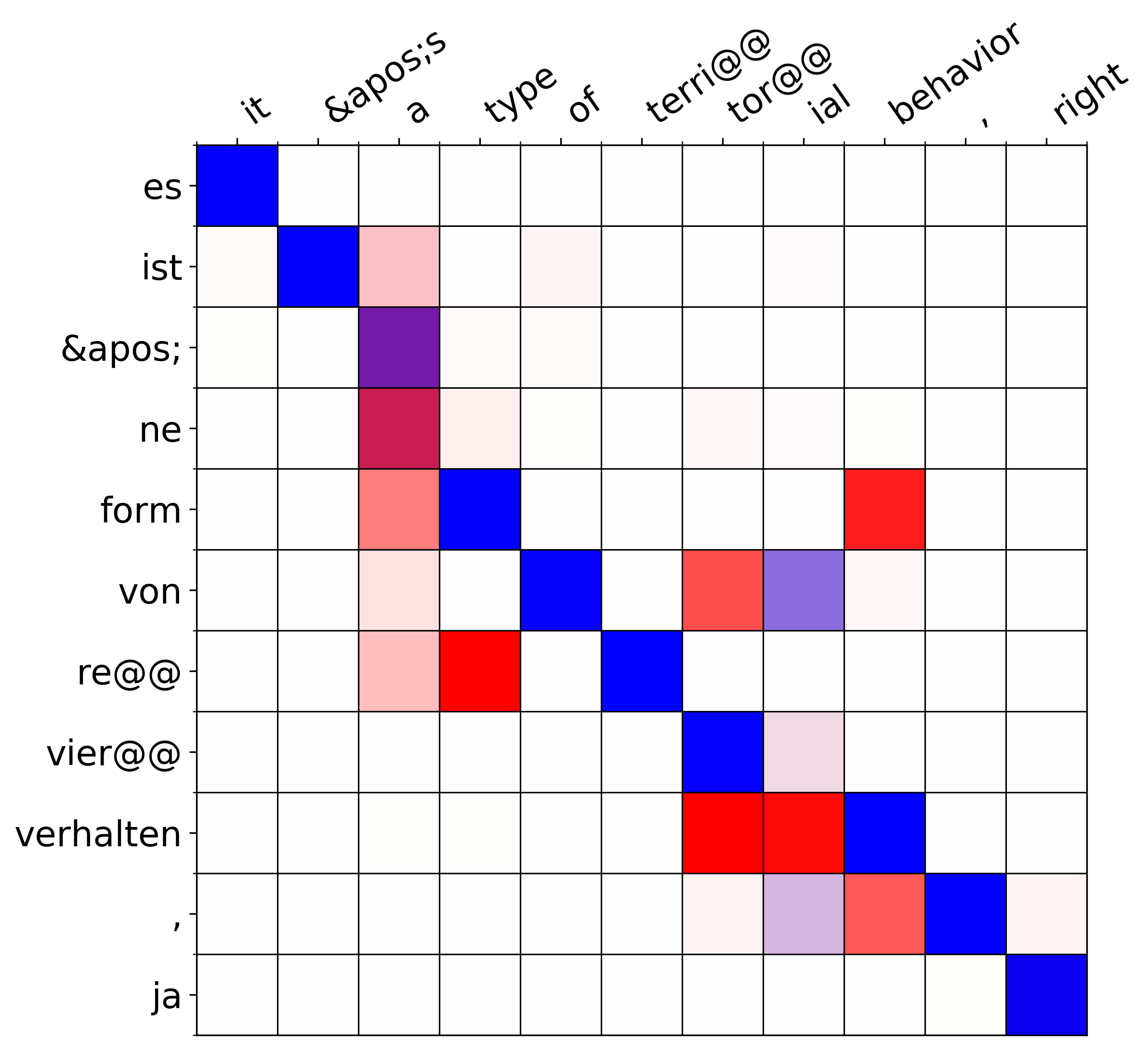}
  \label{fig:a2:e}
  \caption{}
  \end{subfigure}
  \begin{subfigure}[b]{0.49\textwidth}
  \centering
  \includegraphics[height=5cm]{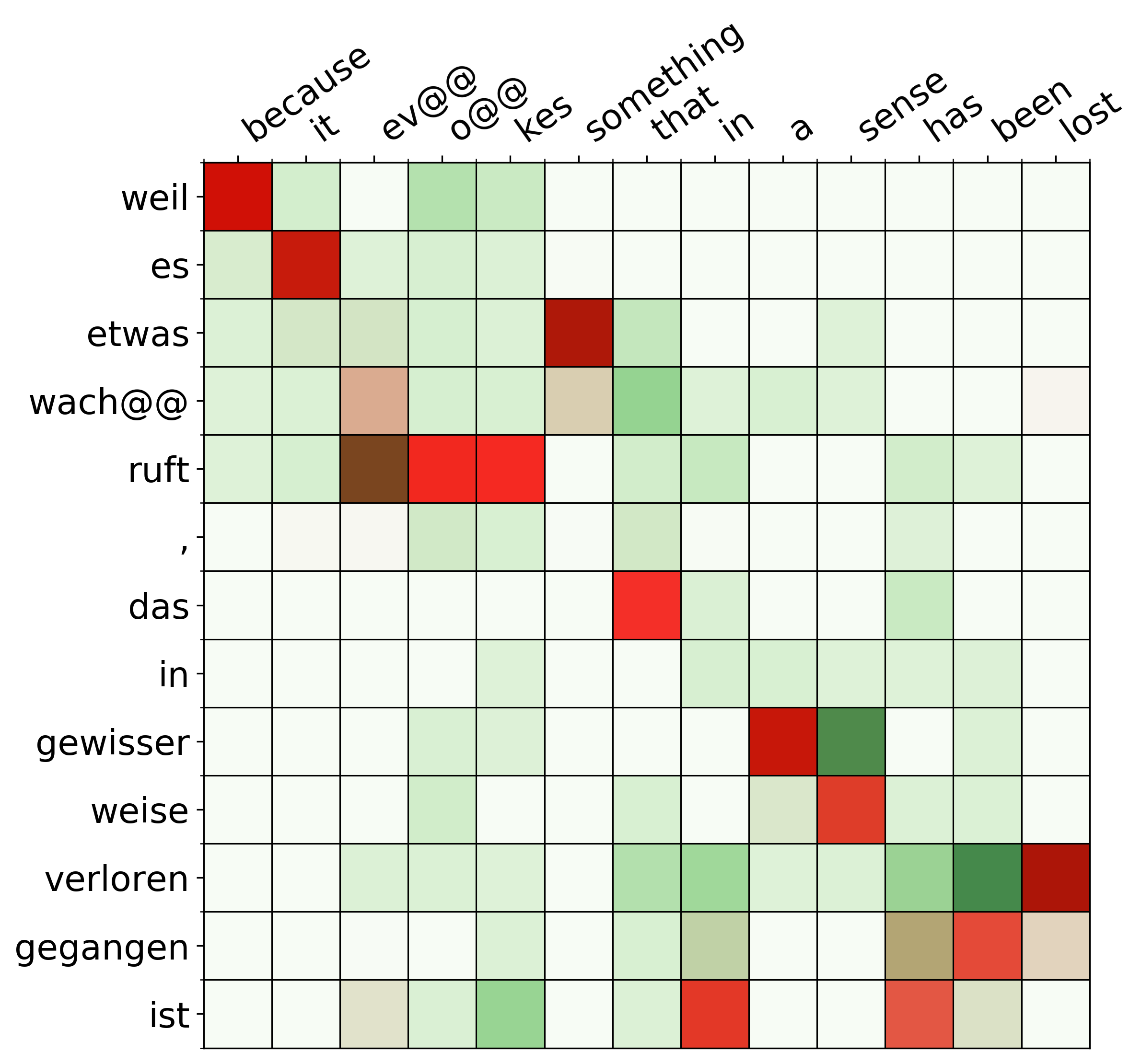}
  \label{fig:a2:f}
  \caption{}
  \end{subfigure}
  
  \caption{\label{fig:attn2} (Left Column) Further examples highlighting the difference between the prior alignment (red) and the variational posterior (blue) when translating from DE-EN (left-to-right). The variational posterior is able to better handle reordering; in (a) the variational posterior successfully aligns `turning' to `verwandelt', in (c) we see a similar pattern with the alignment of the clause `that's my brand' to `das ist meine marke'.
In (e) the prior and posterior both are confused by the `-ial' in `territor-ial', however the posterior still remains more accurate overall and correctly aligns the rest of `revierverhalten' to `territorial behaviour'.
(Right Column) Additional comparisons between soft attention (green) and the prior alignments of variational attention (red). Alignments from both models are similar, but variational attention is lower entropy. Both soft and variational attention rely on aligning the inserted English word `orientation' to the comma in (b) since a direct translation does not appear in the German source.}
\end{figure}

\end{document}